\documentclass{article}
\usepackage{ijcai15}
\usepackage{times}
\usepackage[american]{babel}
\usepackage{graphicx}
\usepackage{epstopdf}
\usepackage[tight,hang]{subfigure}
\usepackage{color,url}
\usepackage{algorithm}
\usepackage{algorithmic}
\usepackage{amsmath,amssymb,amsthm,,float}
\usepackage{hyperref}
\usepackage{enumitem}
\usepackage{mdwmath}
\usepackage{mdwtab}
\newtheorem{thm}{Theorem}

\title{Semi-Orthogonal Multilinear PCA with Relaxed Start}
\author{Qiquan Shi  \textnormal{and}  Haiping Lu\thanks{
This paper will appear in  Proceedings of the 24th International Joint Conference on Artificial Intelligence (IJCAI 2015). \newline 
Copyright \copyright \space 2015, Association for the Advancement of Artificial Intelligence (www.aaai.org). All rights reserved. 
 } \\
Department of Computer Science \\ 
Hong Kong Baptist University,
Hong Kong, China \\
csqqshi@comp.hkbu.edu.hk,
haiping@hkbu.edu.hk}

\begin{document}

\maketitle

\begin{abstract}
Principal component analysis (PCA) is an unsupervised method for learning low-dimensional features with orthogonal projections. Multilinear PCA methods extend PCA to deal with multidimensional data (tensors) directly via tensor-to-tensor projection or tensor-to-vector projection (TVP). However, under the TVP setting, it is difficult to develop an effective multilinear PCA method with the orthogonality constraint. This paper tackles this problem by  proposing a novel \emph{Semi-Orthogonal Multilinear PCA} (SO-MPCA) approach. SO-MPCA learns low-dimensional features directly from tensors via TVP by imposing the orthogonality constraint in  \emph{only one} mode. This formulation results in  more captured variance and more learned features than full orthogonality. For better generalization, we further introduce a  \emph{relaxed start} (RS) strategy to get \textit{SO-MPCA-RS} by fixing the starting projection vectors, which increases the bias and reduces the variance of the learning model. Experiments on both face (2D) and gait (3D) data demonstrate that SO-MPCA-RS outperforms other competing algorithms on the whole, and the relaxed start strategy is also effective for other TVP-based PCA methods.
\end{abstract}

\section{Introduction}
Principal component analysis (PCA) is a classical unsupervised dimensionality reduction method \cite{pcabook}. It transforms input data into a new feature space of lower dimension via \emph{orthogonal projections}, while keeping most variance of the original data. PCA is widely used in areas such as data compression  \cite{kusner2014stochastic}, computer vision \cite{ke2004pca}, and pattern recognition \cite{anaraki2014memory,deng2014transform}.

Many real-world data are multi-dimensional, in the form of \emph{tensors} rather than vectors \cite{KoldaReview09}. The number of dimensions of a tensor is the \emph{order} and each dimension is a \emph{mode} of it. For example,  gray images are second-order tensors (matrices)  and  video sequences are third-order tensors \cite{lu2013multilinear}. Tensor data are also common in applications such as data center monitoring, social network analysis, and network forensics \cite{KoldaICML07TensorTut}. However, PCA on multi-dimensional data requires reshaping tensors into vectors first. This \emph{vectorization} often leads to breaking of original data structures, more complex model with lots of parameters, and high computational and memory demands \cite{lu2013multilinear}. Many researchers address this problem via multilinear extensions of PCA to deal with tensors directly, and there are two main approaches.

One approach is based on \textit{Tensor-to-Tensor Projection} (TTP) that learns low-dimensional \emph{tensors} from high-dimensional \emph{tensors}. The two-dimensional PCA (2DPCA) \cite{2DPCApami04} is probably the first PCA extension to deal with images without vectorization. The generalized low rank approximation of matrices (GLRAM) \cite{GPCAML05} and the generalized PCA (GPCA) \cite{GPCAYe} further generalize 2DPCA from single-sided projections to two-sided projections via reconstruction error minimization and variance maximization, respectively.  Concurrent subspace analysis (CSA) \cite{ConcurrentCVPR2005} and multilinear PCA (MPCA) \cite{MPCATNN} extend GLRAM and GPCA to general higher-order tensors, respectively.

Another approach is based on \textit{Tensor-to-Vector Projection} (TVP) that learns low-dimensional \emph{vectors} from high-dimensional \emph{tensors} in a \emph{successive} way. The tensor rank-one decomposition (TROD) \cite{TRODCVPR2001} minimizes reconstruction error via (greedy) successive residue calculation. The uncorrelated multilinear PCA (UMPCA) \cite{lu2009uncorrelated} maximizes variance with the zero-correlation constraint, following the successive derivation of PCA. However, the number of features that can be extracted by UMPCA is upper-bounded by the lowest mode dimension. For example, for a tensor of size 300$\times $200$\times $3, UMPCA can only extract \emph{three} features, which have very limited usage.

Orthogonality constraint is popular in feature extraction \cite{hua2007face,kokiopoulou2007orthogonal,gao2013stable},  tensor decomposition \cite{kolda2001orthogonal}, and low-rank tensor approximation \cite{edelman1998geometry,wang2015orthogonal}. PCA also obtains orthogonal projections, and the TTP-based PCA methods produce orthogonal projection vectors in each mode. However, none of the existing TVP-based PCA methods derive orthogonal projections. Our study found that it is indeed ineffective to impose full orthogonality in all the modes for TVP-based PCA, due to \emph{low captured variance} and \emph{limited number of extracted features}.

In this paper, we present a new TVP-based multilinear PCA algorithm, \textit{Semi-Orthogonal Multilinear PCA} (SO-MPCA) with \textit{Relaxed Start}, or SO-MPCA-RS, to be detailed in Sec. 3.  There are two main contributions:
\begin{itemize}
\item We propose a novel SO-MPCA approach to maximize the captured variance via TVP with orthogonality constraint in only one mode, which is called \textit{semi-orthogonality} according to \cite{wang2015orthogonal}. The semi-orthogonality results in  \emph{more captured variance} and \emph{more learned features} than full-orthogonality. For the same tensor of size 300$\times$200$\times$3 discussed earlier, SO-MPCA can extract \emph{300} features while full-orthogonal multilinear PCA can only extract \emph{three} features (similar to UMPCA).

\item We introduce a \emph{Relaxed Start (RS)} strategy to get SO-MPCA-RS by fixing the starting projection vectors for better generalization \cite{abu2012learning}. This strategy constrains the hypothesis space to a smaller set, leading to increased bias and reduced variance of the learning model. The experimental results in Sec. 4 show that SO-MPCA-RS outperforms other competing PCA-based methods on the whole. In addition, this is a new strategy for tensor-based algorithms and we show its effectiveness for other TVP-based PCA methods.
\end{itemize}

In the following, we cover the necessary background first.

\section{Background}

\textbf{Notations and basic operations:} We follow the notations in \cite{MultilinearBestR00} to denote vectors by lowercase boldface letters, e.g., $\mathbf{x}$; matrices by uppercase boldface letters, e.g., $\mathbf{X}$; and tensors by calligraphic letters, e.g., $\mathcal{X}$. We denote their elements with indices in parentheses, and indices by lowercase letters spanning the range from 1 to the uppercase letter of the index, e.g., $n =1,\cdots,N$. An $N$th-order tensor $\mathcal{A}\in \mathbb{R}^{I_1\times\cdots\times I_N}$ is addressed by $N$ indices $\{i_n\}$. Each $i_n$ addresses the $n$-mode of $\mathcal{A}$. The $n$-mode product of an $N$th-order 	tensor $\mathcal{A}$ by a vector $\mathbf{u}\in\mathbb{R}^{I_n}$, denoted by $\mathcal{B}=\mathcal{A}\times_n\mathbf{u}^T$,	is a tensor with entries:
\begin{equation}
\mathcal{B}(i_1,\cdots,i_{n-1},1,i_{n+1},\cdots,i_N)=\sum_{i_n}\mathcal{A}(i_1,\cdots,i_N)\cdot\mathbf{u}(i_n).
\end{equation}

\textbf{Tensor-to-vector projection:}	\emph{Elementary multilinear projections }(EMPs) are the building blocks of a TVP. We denote an  EMP as  $\{\mathbf{u}^{(1)}, \mathbf{u}^{(2)},\cdots,\mathbf{u}^{(N)}\}$, consisting of one unit projection vector in each mode, i.e., $\parallel\mathbf{u}^{(n)}\parallel=1$ for $n=1,\cdots,N$, where $\parallel\cdot\parallel$ is the Euclidean norm for vectors. It projects a tensor $\mathcal{X}\in\mathbb{R}^{I_1\times I_2\times\cdots\times I_N}$  to a scalar $y$ through the $N$ unit projection vectors as \cite{lu2013multilinear}:
\begin{eqnarray}y=\mathcal{X}\times_1\mathbf{u}^{(1)^T}\times_2 \mathbf{u}^{(2)^T}\cdots\times_N\mathbf{u}^{(N)^T}.\end{eqnarray}

The TVP of a tensor $\mathcal{X}$ to a vector $\mathbf{y}\in \mathbb{R}^{P}$ consists of $P$ EMPs $\{\mathbf{u}^{(1)}_p,\cdots,\mathbf{u}^{(N)}_p\},p=1,\cdots,P$, which can be written concisely as $\{\mathbf{u}^{(n)}_p, n=1,\cdots,N\}_{p=1}^P$ or $\{\mathbf{u}^{(n)}_p\}_{p=1}^P$:
\begin{equation}\mathbf{y}=\mathcal{X}\times_{n=1}^N\{\mathbf{u}^{(n)}_p,n=1,\cdots,N\}_{p=1}^P,\end{equation}
where the $p$th component of $\mathbf{y}$ is obtained from the $p$th EMP as:
\begin{equation}y_p=\mathbf{y}(p)=\mathcal{X}\times_1 \mathbf{u}^{(1)^T}_p\cdots\times_N\mathbf{u}^{(N)^T}_p=\mathcal{X}\times_{n=1}^N\{\mathbf{u}^{(n)}_p \}.
\end{equation}

\section{SO-MPCA with Relaxed Start}
This section presents the proposed SO-MPCA-RS by first formulating the SO-MPCA problem, then deriving the solutions with a successive and conditional approach, and finally introducing the relaxed start strategy for better generalization.

\subsection{Formulation of Semi-Orthogonal MPCA}
We define the SO-MPCA problem with orthogonality constraint in only one mode, i.e., \emph{semi-orthogonality}~\cite{wang2015orthogonal}, as follows:

\textbf{The SO-MPCA problem:}	A set of $M$ tensor data samples \{$\mathcal{X}_1$, $\mathcal{X}_2, \cdots, \mathcal{X}_M$\} are available for training. Each sample $\mathcal{X}_m\in \mathbb{R}^{I_1\times I_2\times\cdots\times I_N}$ can be viewed a point in the tensor space $\mathbb{R}^{I_1}\bigotimes \mathbb{R}^{I_2}\cdots\bigotimes \mathbb{R}^{I_N}$, where $I_n$ is the $n$-mode dimension and $\bigotimes$ denotes the Kronecker product. SO-MPCA considers a TVP, which consists of $P$ EMPs  $\{\mathbf{u}_p^{(n)}\in \mathbb{R}^{I_n\times 1},n=1,\cdots,N\}_{p=1}^P$, that projects the input tensor space  $\mathbb{R}^{I_1}\bigotimes \mathbb{R}^{I_2}\cdots\bigotimes \mathbb{R}^{I_N}$ into a vector subspace $\mathbb{R}^{P}$, i.e.,
\begin{equation}\label{Y}
\mathbf{y}_m=\mathcal{X}_m\times_{n=1}^N\{\mathbf{u}^{(n)}_p,n=1,\cdots,N\}_{p=1}^P
\end{equation}
for $m=1,\cdots,M$. The objective is to find a TVP to maximize the variance of the projected samples in each projection direction, subject to the orthogonality constraint in \emph{only one} mode, denoted as the $\nu$-mode.  The variance is measured by the total scatter $S_{p}$ defined as:
\begin{equation} \label{Scatter}
S_p=\sum_{m=1}^M (y_{m_p}-\bar{y}_p)^2,
\end{equation}
where $y_{m_p}=\mathcal{X}_m\times_{n=1}^N\{\mathbf{u}^{(n)}_p \}$, and $\bar{y}_p=\frac{1}{M}\sum_my_{m_p}$.

In other words, the objective of SO-MPCA is to obtain the $P $ EMPs, with the $p$th EMP determined as:
\begin{eqnarray}
\{\mathbf{u}^{(n)}_p, n=1,\cdots,N\}& =& \arg\max\sum_{m=1}^M(y_{m_p}-\overline{y}_{p})^2, \hspace{0.28cm} \label{SOMPCA2}\\
\mathrm{s.t.}\:\:
\mathbf{u}^{{(n)^T}}_p\mathbf{u}^{{(n)}}_p=1\:&\mathrm{for}&\: n=1,\cdots,N \:\:\mathrm{and}\: \label{SOMPCA1}     \hspace{0.39cm}  \\
\mathbf{u}^{(\nu)^T}_p \mathbf{u}^{(\nu)}_q = 0 \:& \mathrm{for}&\: p> 1 \:\mathrm{and}\: q = 1,\cdots,p-1, \hspace{0.39cm}\label{SOMPCA11} \end{eqnarray}
where the orthogonality constraint (\ref{SOMPCA11}) is imposed only in the $\nu$-mode and there is no such constraint for the other modes ($n= 1,\cdots,N$, $n \neq\nu$). The normalization constraint (\ref{SOMPCA1}) is imposed for all modes.

\textbf{Bound on the number of features:}
Based on the proof of Corollary 1 in \cite{lu2009uncorrelated}, we can derive that the number of features $P$ that can be extracted by SO-MPCA is upper-bounded by the $\nu$-mode dimension $I_\nu$: $P  \leq I_\nu$. Since we can choose any $n$ as $\nu$, we have the upper bound of $P$ as $P\leq \max_n I_n$ (i.e., the highest mode dimension).

\textbf{Selection of mode $\nu$:} Although we are free to choose any mode $n$ as $\nu$ to impose the orthogonality constraint (\ref{SOMPCA11}), it is often good to have more features in practice. Thus, in this paper, we choose the mode with the highest dimension as $\nu$:
\begin{equation} 
\nu = \arg \max_n I_n, \label{determineNU}
\end{equation}
such that $P= \max_n I_n = I_{\nu}$. On the other hand, we can also obtain a total of $\sum_n I_n$ features by running SO-MPCA $N$ times with $\nu=1,\cdots,N$. In this paper, we only focus on SO-MPCA with $\nu$ determined by (\ref{determineNU}).

\textbf{Semi-orthogonality vs. full-orthogonality:} If we impose the orthogonality constraint (\ref{SOMPCA11}) in \emph{all} modes, we can get \textit{Full-Orthogonal Multilinear PCA} (FO-MPCA). However, our study found that FO-MPCA is not effective primarily due to two reasons:
\begin{itemize}
\item Due to the heavy constraints, the variance captured by FO-MPCA is quite low, even lower than UMPCA. In contrast, SO-MPCA can capture more variance than both FO-MPCA and UMPCA. This is illustrated in Fig. \ref{Facetest1} in Sec. 4.
\item Similar to UMPCA, the number of features that can be extracted by FO-MPCA is upper-bounded by the lowest mode dimension $\min_n I_n$, which can be quite limited. For instance, FO-MPCA can extract only \emph{three} features for a tensor of size 300$\times $200$\times $3 while SO-MPCA can extract \emph{300} features by choosing $\nu=1$ for the same tensor. This can be observed in Fig. \ref{Facetest1} as well.
\end{itemize}

\subsection { Successive Derivation of SO-MPCA }
To solve the SO-MPCA problem, we follow the successive derivation in \cite{pcabook,lu2009uncorrelated} to determine EMPs one by one in $P$ steps:	
\begin{description}
\item [\textbf{Step $1 \hspace{.15cm}(p = 1) $:}]
Determine the first EMP $\{\mathbf{u}^{(n)}_1, n = 1,\cdots,N\}$ by maximizing $S_{1}$ with the constraint (\ref{SOMPCA1}).
 \item [\textbf{Step $p \hspace{.15cm}  (p=2,\cdots,P)$:}] \hspace{.1cm} \\
   \hspace{.3cm}  Determine the $p$th EMP $\{\mathbf{u}^{(n)}_p,  n = 1,\cdots,N\}  $  by maximizing $S_{p} $ with the constraints (\ref{SOMPCA1}) and (\ref{SOMPCA11}).
\end{description}

\textbf{Conditional subproblem: } In order to obtain the $p$th EMP $\{\mathbf{u}^{(n)}_p,n=1,\cdots,N\}$, we need to determine $N$ vectors. We follow the approach of alternating least squares \cite{PARAFAC70}. Thus, we can only obtain \emph {locally optimal }solutions as in many other tensor-based methods. For the $p$th EMP, the parameters of the $n$-mode projection vector $\mathbf{u}^{(n)}_p$ are estimated one mode by one mode separately conditioned on the projection vectors in all the other modes. Assuming the $p$th projection vectors in all but $n$-mode are given, we project the input tensor samples in these $(N-1)$ modes to obtain the partial multilinear projections as in \cite{lu2013multilinear}:
\begin{eqnarray}\tilde{\mathbf{y}}_{m_p}^{(n)}=&&\mathcal{X}_m\times_1 \mathbf{u}^{(1)^T}_p\cdots\times_{n-1}
\mathbf{u}^{(n-1)^T}_p\nonumber\\&&\times_{n+1} \mathbf{u}^{(n+1)^T}_p\cdots\times_N\mathbf{u}^{(N)^T}_p , \label{partialY} \end{eqnarray}
where $\tilde{\mathbf{y}}_{m_p}^{(n)}\in \mathbb{R}^{I_{n}}$. This conditional subproblem then becomes to determine $\mathbf{u}^{(n)}_p$ that projects the vector samples $\{\tilde{\mathbf{y}}_{m_p}^{(n)}, m=1,\cdots,M\}$ onto a line to maximize the variance captured. Then the total scatter matrix $\tilde{\mathbf{S}}_{p}^{(n)}$ corresponding to  $\{\tilde{\mathbf{y}}_{m_p}^{(n)}, m=1,\cdots,M\}$ becomes:
\begin{eqnarray}\label{STPN}
\tilde{\mathbf{S}}_{p}^{(n)}&=&\sum_{m=1}^M
(\tilde{\mathbf{y}}_{m_p}^{(n)}-\bar{\tilde{\mathbf{y}}}_{p}^{(n)})(\tilde{\mathbf{y}}_{m_p}^{(n)}-\bar{\tilde{\mathbf{y}}}_{p}^{(n)})^T,\label{nmST1}
\end{eqnarray}
where $\bar{\tilde{\mathbf{y}}}^{(n)}_p=\frac{1}{M}\sum^M_{m=1}\tilde{\mathbf{y}}_{m_p}^{(n)}$.

For $p=1$ (step 1), the solution for $\mathbf{u}^{(n)}_1$, where $n=1,\cdots,N$, is obtained as the unit eigenvector of $\tilde{\mathbf{S}}_{1}^{(n)}$ associated with the largest eigenvalue.

For $p\geq 2$, we need to deal with the $\nu$-mode and other modes differently. For modes other than $\nu$, the solution for $\mathbf{u}^{(n)}_p$, where $n=1,\cdots,N$, $n\neq \nu$, is obtained as the unit eigenvector of $\tilde{\mathbf{S}}_{p}^{(n)}$ associated with the largest eigenvalue.

\textbf{Constrained optimization for $\nu$-mode and $p\geq 2$:} When $p\geq 2$, we need to determine $\mathbf{u}^{(\nu)}_p$ by solving the following constrained optimization problem:
\begin{eqnarray} \label{SOMPCA33}
\mathbf{u}^{(\nu)}_p
=\arg\max\mathbf{u}^{{(\nu)^T}}_p\tilde{\mathbf{S}}_{p}^{(\nu)}\mathbf{u}^{(\nu)}_p&&   \\
\mathrm{s.t. }\:
\mathbf{u}^{{(\nu)^T}}_p\mathbf{u}^{{(\nu)}}_p=1\:\mathrm{and}\:
\mathbf{u}^{(\nu)^T}_p \mathbf{u}^{(\nu)}_q = 0,&&\hspace{-.4cm} q =1,\cdots p-1. \nonumber	
\end{eqnarray}

We solve this problem by the following theorem:

\begin{thm}\label{thmSOMPCA}
The solution to the problem (\ref{SOMPCA33}) is the (unit-length) eigenvector corresponding to the largest eigenvalue of the following eigenvalue problem:
\begin{equation}\label{SOMPCAfinal}
 \begin{aligned}
 \boldsymbol\Gamma_p^{(\nu)} \hspace{.05cm} \tilde{\mathbf{S}}^{(\nu)}_{p} \hspace{.1cm} \mathbf{u}_{p}^{(\nu)}  \hspace{.1cm} =  \hspace{.1cm}  \lambda  \mathbf{u}_{p}^{(\nu)} ,
 \end{aligned}
\end{equation}
where,
\begin{eqnarray}\label{Gammap}
\boldsymbol\Gamma_p^{(\nu)} =  \hspace{.1cm} \hspace{.1cm} [ \hspace{.1cm} \mathbf{I}_{I_n}\hspace{.1cm} - \hspace{.1cm}   \hspace{.2cm}  \sum_{q=1}^{p-1}  \mathbf{u}_{q}^{(\nu)} {\mathbf{u}_{q}^{(\nu)} }^T  \hspace{.1cm}], &&
 \end{eqnarray}
and $\mathbf{I}_{I_n}$ is an identity matrix of size ${I_n} \times  {I_n}$.
\end{thm}

\begin{proof} First, we use Lagrange multipliers to transform the problem (\ref{SOMPCA33}) to include all the constraints as:
\begin{eqnarray}\label{laglangri}
\mathcal{L}_{\nu}  &=& {\mathbf{u}_{p}^{(\nu)}}^T  \tilde{\mathbf{S}}^{(\nu)}_{p} \mathbf{u}_{p}^{(\nu)}  -\lambda ({\mathbf{u}_{p}^{(\nu)}}^T  \mathbf{u}_{p}^{(\nu)} -1 )
\nonumber\\&-& \sum_{q=1}^{p-1}\mu_q {\mathbf{u}_{p}^{(\nu)} }^T   \mathbf{u}_{q}^{(\nu)},
\end{eqnarray}
where $\lambda $ and $\{\mu_q,q=1,\cdots,p-1\}$ are Lagrange multipliers.

Then we set the partial derivative of $ \mathcal{L}_{\nu}$ with respect to $\mathbf{u}_{p}^{(\nu)}$ to zero:
\begin{eqnarray}
\centering
\frac{\partial \mathcal{L}_{\nu}  }{\partial\mathbf{u}_{p}^{(\nu)}  }   =  2  \tilde{\mathbf{S}}^{(\nu)}_{{p}} \mathbf{u}_{p}^{(\nu)} -  2 \lambda  \mathbf{u}_{p}^{(\nu)} -  \sum_{q=1}^{p-1} \mu_q  \mathbf{u}_{q}^{(\nu)} =  0 .   \label{OMPCA2}
\end{eqnarray}
Premultiplying (\ref{OMPCA2}) by $ {\mathbf{u}_{p}^{(\nu)}}^T$, the third term vanishes and we get
\begin{eqnarray}\label{Lambda}
      \begin{aligned}
     2 {\mathbf{u}_{p}^{(\nu)} }^T  \tilde{\mathbf{S}}^{(\nu)}_{p} \mathbf{u}_{p}^{(\nu)} -  2 \lambda  {\mathbf{u}_{p}^{(\nu)} }^T \mathbf{u}_{p}^{(\nu)}  = 0 &&\\
      \Rightarrow \hspace{.3cm} \lambda =  {\mathbf{u}_{p}^{(\nu)} }^T  \tilde{\mathbf{S}}^{(\nu)}_{p} \mathbf{u}_{p}^{(\nu)},  \hspace{2.1cm}
      \end{aligned}
    \end{eqnarray}
which indicates that $\lambda$ is exactly the criterion to be maximized, with the orthogonality constraint.

Next, a set of $(p - 1) $ equations are obtained  by  premultiplying (\ref{OMPCA2}) by ${\mathbf{u}_{q}^{(\nu)}}^T, q =1,\cdots,p-1$, respectively,
\begin{eqnarray}\label{OMPCA4}
 2 {\mathbf{u}_{q}^{(\nu)} }^T  \tilde{\mathbf{S}}^{(\nu)}_{p} \mathbf{u}_{p}^{(\nu)}   -   2 \lambda  {\mathbf{u}_{q}^{(\nu)} }^T \mathbf{u}_{p}^{(\nu)}  -   \sum_{s=1}^{p-1} \mu_s {\mathbf{u}_{q}^{(\nu)} }^T\mathbf{u}_{s}^{(\nu)}   = 0.   		
 \end{eqnarray}
The second term vanishes and the summand in the third term is non-zero only for $s=q$. Thus, we get
\begin{eqnarray}
 2 {\mathbf{u}_{q}^{(\nu)} }^T  \tilde{\mathbf{S}}^{(\nu)}_{p} \mathbf{u}_{p}^{(\nu)} -  \mu_q  = 0   \Rightarrow\: \mu_{q} =    {2{\mathbf{u}_{q}^{(\nu)} }^T  \tilde{\mathbf{S}}^{(\nu)}_{p}} \mathbf{u}_{p}^{(\nu)}.\label{OMPCA45}
\end{eqnarray}
Substituting (\ref{OMPCA45}) into (\ref{OMPCA2}), we get
  \begin{eqnarray}
   2  \tilde{\mathbf{S}}^{(\nu)}_{p} \mathbf{u}_{p}^{(\nu)}  -
   2 \lambda  \mathbf{u}_{p}^{(\nu)}  - \sum_{q=1}^{p-1} \mathbf{u}_{q}^{(\nu)}\cdot  {2{\mathbf{u}_{q}^{(\nu)} }^T  \tilde{\mathbf{S}}^{(\nu)}_{p} \mathbf{u}_{p}^{(\nu)}}    = 0   \nonumber  \\
     \Rightarrow \hspace{.3cm}  \lambda  \mathbf{u}_{p}^{(\nu)}  =   \tilde{\mathbf{S}}^{(\nu)}_{p} \mathbf{u}_{p}^{(\nu)} -    \sum_{q=1}^{p-1}  \mathbf{u}_{q}^{(\nu)}  {\mathbf{u}_{q}^{(\nu)} }^T  \tilde{\mathbf{S}}^{(\nu)}_{p} \mathbf{u}_{p}^{(\nu)}        \\
        \Rightarrow \hspace{.4cm}
     \lambda  \mathbf{u}_{p}^{(\nu)}    =   [ \hspace{.1cm} \mathbf{I}_{I_n} -   \sum_{q=1}^{p-1}  \mathbf{u}_{q}^{(\nu)} {\mathbf{u}_{q}^{(\nu)} }^T  \hspace{.1cm} ] \hspace{.2cm} \tilde{\mathbf{S}}^{(\nu)}_{p} \mathbf{u}_{p}^{(\nu)} .  \label{OMPCA5}
  \end{eqnarray}
Using the definition in (\ref{Gammap}),  (\ref{OMPCA5}) can be rewritten as:
\begin{equation}
   \begin{aligned}
     \boldsymbol\Gamma_p^{(\nu)} \hspace{.05cm} \tilde{\mathbf{S}}^{(\nu)}_{p} \hspace{.1cm} \mathbf{u}_{p}^{(\nu)}  \hspace{.1cm} =  \hspace{.1cm}  \lambda  \mathbf{u}_{p}^{(\nu)}.
     \end{aligned}    \label{SOMPCAsolution}
  \end{equation}
  Since  $\lambda$ is the criterion to be maximized, this maximization is achieved by setting  $\mathbf{u}_{p}^{(\nu)}$  to the (unit) eigenvector of  $ \boldsymbol\Gamma_p^{(\nu)} \hspace{.05cm} \tilde{\mathbf{S}}^{(\nu)}_{p} $ associated with  its corresponding largest eigenvalue . 	
\end{proof}

\begin{algorithm}[!t]
\small
\caption{Semi-Orthogonal Multilinear PCA with Relaxed Start (SO-MPCA-RS) }
\begin{algorithmic}[1]
\STATE {\bfseries Input:}
A set of tensor samples $ \{\mathcal{X}_m \in \mathbb{R}^{I_1\times  \cdots\times I_N},m=1,\cdots,M\}$, and the maximum number of iterations $K$.
\STATE Set $\nu = \arg \max_n I_n 	$.
\STATE Set the first EMP: $\mathbf{u}^{(n)}_{1}=\mathbf{1}/\parallel\mathbf{1}\parallel$ for $ n=1,\cdots,N$ .

\FOR{$p=2$ {\bfseries to} $P$}
\STATE Initialize  $\mathbf{u}^{(n)}_{p}=\mathbf{1}/\parallel\mathbf{1}\parallel$ for $ n=1,\cdots,N$.

\FOR{$k=1$ {\bfseries to} $K$}
\FOR{$n=1$ {\bfseries to} $N$}
  \STATE  Calculate the partial multilinear projection $\{\tilde{\mathbf{y}}_{m_p}^{(n)}\}$ for $m=1,\cdots,M$ according to (\ref{partialY}).
  \IF {$n==\nu$}
\STATE Calculate $ \boldsymbol\Gamma_p^{(\nu)}$ and $\tilde{\mathbf{S}}_{p}^{(\nu)}$ according to (\ref{Gammap}) and (\ref{STPN}), respectively. Then, set $\mathbf{u}^{(\nu)}_{p}$ to the eigenvector of
  $\boldsymbol\Gamma_p^{(\nu)} \tilde{\mathbf{S}}_{p}^{(\nu)}$ associated with
  the largest eigenvalue.
   \ELSE
  \STATE Calculate  $\tilde{\mathbf{S}}_{p}^{(n)}$ by (\ref{STPN}). Set $\mathbf{u}^{(n)}_{p}$ to the eigenvector of $ \tilde{\mathbf{S}}_{p}^{(n)}$ associated with  the largest eigenvalue.
\ENDIF

\ENDFOR
\ENDFOR
  \ENDFOR

\STATE \textbf{Output}  The TVP  $\{\mathbf{u}^{(n)}_p, n=1,\cdots,N \}_{p=1}^P$ .
\end{algorithmic}\label{SOMPCARS}
\end{algorithm}

\subsection{Relaxed Start for Better Generalization}
When we use SO-MPCA features for classification, we find the performance is limited. Therefore, we further introduce a simple \emph{relaxed start} (RS) strategy to get SO-MPCA-RS by fixing the first EMP $\{\mathbf{u}^{(n)}_1, n=1,\cdots,N\}$ (the starting projection vectors), without variance maximization. In this paper, we set this starting EMP $\mathbf{u}^{(n)}_1$ (for $n=1,\cdots,N)$ to the normalized uniform vector $\mathbf{1}/\parallel\mathbf{1}\parallel$ for simplicity.

This idea is motivated by the theoretical studies in Chapter 4 of \cite{abu2012learning} showing that constraining a learning model could lead to better generalization. By fixing the first EMP as simple vectors, the following EMPs have less freedom due to the imposed semi-orthogonality, which increases the bias and reduces the variance of the learning model. Thus, the SO-MPCA-RS model has a smaller hypothesis set than the SO-MPCA model. The two algorithms differ only in how to determine the first (starting) EMP though the following EMPs will all be different due to their dependency on the first EMP. 

This relaxed start strategy is not specific to SO-MPCA but generally applicable to any TVP-based subspace learning algorithm. We run controlled experiments in Sec. 4 to show that it can improve the performance of not only SO-MPCA but also TROD and UMPCA.

Algorithm 1 summarizes the SO-MPCA-RS algorithm.\footnote{Matlab code is available at: \url{http://www.comp.hkbu.edu.hk/~haiping/codedata.html}} The SO-MPCA algorithm can be obtained from Algorithm 1 by removing line 3, changing $p=2$ in line 4 to $p=1$ and setting $\boldsymbol\Gamma_1^{(\nu)}$ ($p=1$) in line 10 to an identity matrix.
\begin{table*}[!ttt]\caption{Face recognition rates in percentage (mean  $\pm$ std) by the nearest neighbor classifier  on the FERET subset. The top two results are highlighted with bold fonts and `-' indicates that no enough features can be extracted.}
\centerline{
\scriptsize
\centering
\setlength{\tabcolsep}{7.5pt}
\renewcommand{\arraystretch}{1.45}
\begin{tabular}{|c|c|ccccc|cc| c c|}
\hline $L$ & $P$   & PCA     &   CSA &   MPCA &  TROD &  UMPCA & SO-MPCA     & \textbf{SO-MPCA-RS }       & TROD-\textbf{RS }        & UMPCA-\textbf{RS} \\\hline%
&1	 &2.60$\pm$0.66	&3.87$\pm$1.02	&2.52$\pm$0.76	&2.70$\pm$0.44	&5.98$\pm$2.65	& 2.73 $\pm$0.69	&\textbf{6.85}$\pm$1.44	&2.63$\pm$0.82	&\textbf{6.04}$\pm$2.00		\\
&5	 &15.12$\pm$1.31	        &11.90$\pm$1.25	&16.65$\pm$1.82	&15.88$\pm$1.20	&23.23$\pm$4.49	&20.06$\pm$2.34	&\textbf{27.27}$\pm$2.36	&16.68$\pm$1.50	&\textbf{24.78}$\pm$4.76		\\
&10	 &22.69$\pm$2.21	        &21.35$\pm$2.76	&22.24$\pm$1.94	&21.52$\pm$2.83	&31.83$\pm$5.17	&28.77$\pm$2.72	&\textbf{36.34}$\pm$3.56	&21.86$\pm$3.03	&\textbf{33.16}$\pm$5.36		\\
$1$ &20	 &{27.62}$\pm$2.57	&26.16$\pm$2.38	&27.16$\pm$1.47	&26.30$\pm$2.49	&35.94$\pm$5.65	&31.94$\pm$2.95	&\textbf{40.32}$\pm$3.40	&26.51$\pm$1.97	&\textbf{36.65}$\pm$5.46		\\
&50	 &{31.38}$\pm$2.58	&31.37$\pm$1.92	&31.29$\pm$1.71	&29.63$\pm$2.21	&36.14$\pm$5.73	&32.33$\pm$2.78	&\textbf{40.48}$\pm$3.09	&29.80$\pm$1.48	&\textbf{37.05}$\pm$5.46		\\
&80	&--&31.95$\pm$1.84	&{32.17}$\pm$2.09	&31.14$\pm$2.42	 & --	&\textbf{32.26}$\pm$2.71	&\textbf{40.41}$\pm$3.09	&31.21$\pm$1.94		&	 --  \\\hline
&1	   &2.69$\pm$0.73	&3.36$\pm$0.54	&2.63$\pm$0.55	&2.65$\pm$0.77	&\textbf{7.28}$\pm$2.44	&2.69$\pm$0.46	&\textbf{7.97}$\pm$1.10	&2.81$\pm$0.79	&6.82$\pm$1.56		\\
&5  	&20.17$\pm$1.25	&15.15$\pm$1.03	&21.53$\pm$0.90	&21.34$\pm$1.56	&26.90$\pm$5.23	&24.34$\pm$1.59	&\textbf{33.82}$\pm$1.45	&21.62$\pm$1.93	&\textbf{29.19}$\pm$3.55		\\
&10	    &32.03$\pm$2.49	&29.45$\pm$1.95	&28.04$\pm$1.69	&30.05$\pm$1.70	&40.17$\pm$6.76	&36.94$\pm$2.58	&\textbf{46.63}$\pm$1.95	&30.83$\pm$1.93	&\textbf{44.01}$\pm$3.22		\\
$2$&20	&39.07$\pm$1.87	&37.07$\pm$2.27	&38.86$\pm$2.12	&36.87$\pm$2.37	&44.51$\pm$6.54	&41.70$\pm$2.48	&\textbf{52.19}$\pm$2.11	&37.64$\pm$2.19	&\textbf{48.67}$\pm$3.57		\\
&50	&43.86$\pm$2.53	&44.61$\pm$2.34	&44.54$\pm$2.74	&42.67$\pm$2.16	&45.47$\pm$6.65	&42.24$\pm$2.39	&\textbf{52.22}$\pm$1.73	&42.99$\pm$2.28	&\textbf{49.07}$\pm$3.61		\\
&80	&45.28$\pm$2.39	&45.82$\pm$2.76	&\textbf{46.02}$\pm$2.67	&44.46$\pm$2.53	&-- 	&42.20$\pm$2.39	&\textbf{52.19}$\pm$1.80	&44.78$\pm$2.48		&	--   \\\hline
&1	&2.72$\pm$0.45	&4.07$\pm$0.80	&2.25$\pm$0.44	&2.95$\pm$0.61	&\textbf{7.42}$\pm$1.17	&2.56$\pm$0.56	&\textbf{7.55}$\pm$1.17	&2.74$\pm$0.81	&7.34$\pm$1.08		\\
&5	&23.89$\pm$1.64	&16.58$\pm$0.95	&25.95$\pm$1.26	&24.48$\pm$1.79	&\textbf{33.86}$\pm$3.65	&28.30$\pm$2.05	&\textbf{36.93}$\pm$1.76	&26.14$\pm$2.17	&32.68$\pm$4.67		\\
&10	&37.20$\pm$1.91	&36.05$\pm$1.50	&34.91$\pm$2.40	&34.83$\pm$2.96	&49.39$\pm$2.83	&43.31$\pm$1.51	&\textbf{54.38}$\pm$3.09	&34.64$\pm$3.01	&\textbf{50.55}$\pm$4.81		\\
$3$&20	&46.05$\pm$2.11	&43.87$\pm$1.98	&45.48$\pm$2.46	&43.37$\pm$2.51	&55.83$\pm$3.32	&49.49$\pm$2.16	&\textbf{61.25}$\pm$2.85	&42.99$\pm$2.65	&\textbf{55.89}$\pm$4.16		\\
&50	&51.35$\pm$2.53	&51.60$\pm$2.50	&52.00$\pm$2.70	&48.92$\pm$2.69	&56.42$\pm$3.11	&49.80$\pm$2.29	&\textbf{61.08}$\pm$2.83	&49.47$\pm$2.76	&\textbf{56.38}$\pm$4.62		\\
&80	&52.66$\pm$2.67	&52.84$\pm$2.70	&\textbf{53.31}$\pm$2.59	&51.17$\pm$2.65	&-- 	&49.77$\pm$2.36	&\textbf{60.98}$\pm$2.83	&51.45$\pm$2.63		&-- 	  \\\hline
&1	&2.68$\pm$0.85	&3.92$\pm$1.04	&2.22$\pm$0.62	&3.11$\pm$0.83	&\textbf{7.96}$\pm$2.15	&3.13$\pm$0.90	&\textbf{8.34}$\pm$1.37	&3.08$\pm$0.66	&7.03$\pm$1.68		\\
&5	&25.26$\pm$1.77	&18.93$\pm$1.29	&28.71$\pm$1.91	&27.37$\pm$2.34	&37.66$\pm$4.88	&29.61$\pm$2.16	&\textbf{40.25}$\pm$1.52	&28.25$\pm$2.82	&\textbf{38.84}$\pm$2.97		\\
&10	&41.54$\pm$2.02	&40.39$\pm$2.36	&39.43$\pm$2.05	&38.82$\pm$3.91	&55.10$\pm$4.55	&47.10$\pm$2.88	&\textbf{59.30}$\pm$2.49	&39.25$\pm$2.95	&\textbf{57.30}$\pm$4.86		\\
$4$&20	&49.34$\pm$1.69	&49.39$\pm$2.35	&50.18$\pm$3.03	&47.57$\pm$2.70	&62.40$\pm$4.49	&53.85$\pm$2.89	&\textbf{66.60}$\pm$3.07	&47.73$\pm$2.96	&\textbf{64.08}$\pm$4.94		\\
&50	&56.85$\pm$2.09	&57.51$\pm$2.98	&57.48$\pm$2.72	&54.58$\pm$2.56	&63.13$\pm$4.15	&54.56$\pm$3.14	&\textbf{67.03}$\pm$2.86	&54.42$\pm$2.89	&\textbf{64.85}$\pm$5.01		\\
&80	&58.16$\pm$2.46	&\textbf{59.05}$\pm$2.65	&{58.91}$\pm$2.50	&57.30$\pm$2.46	&-- 	&54.47$\pm$3.17	&\textbf{66.89}$\pm$2.90	&57.48$\pm$2.46		&	 --  \\\hline
&1	&2.91$\pm$0.91	&4.37$\pm$1.06	&2.72$\pm$0.88	&2.59$\pm$0.64	&\textbf{7.41}$\pm$2.14	&3.07$\pm$0.60	&\textbf{8.38}$\pm$0.97	&2.96$\pm$0.75	&7.20$\pm$1.35		\\
&5	&28.95$\pm$2.07	&20.75$\pm$1.77	&32.99$\pm$2.47	&31.75$\pm$2.79	&40.78$\pm$5.82	&34.20$\pm$2.67	&\textbf{42.35}$\pm$3.04	&33.45$\pm$1.41	&\textbf{41.67}$\pm$1.95		\\
&10	&47.06$\pm$1.54	&45.77$\pm$2.17	&43.29$\pm$3.07	&43.80$\pm$3.51	&60.49$\pm$6.37	&53.45$\pm$2.75	&\textbf{63.23}$\pm$3.37	&44.69$\pm$2.71	&\textbf{62.88}$\pm$2.59		\\
$5$&20	&55.66$\pm$1.94	&56.01$\pm$2.19	&56.79$\pm$2.14	&54.47$\pm$1.66	&66.90$\pm$6.23	&61.40$\pm$2.43	&\textbf{69.97}$\pm$2.55	&54.64$\pm$1.96	&\textbf{70.19}$\pm$3.25		\\
&50	&63.91$\pm$1.71	&64.58$\pm$2.13	&64.37$\pm$2.27	&61.54$\pm$2.75	&67.71$\pm$6.31	&62.26$\pm$2.88	&\textbf{70.70}$\pm$2.47	&61.51$\pm$1.92	&\textbf{70.81}$\pm$3.08		\\
&80	&64.61$\pm$1.67	&65.58$\pm$1.92	&\textbf{65.85}$\pm$2.02	&64.02$\pm$2.40	&-- 	&62.18$\pm$2.83	&\textbf{70.70}$\pm$2.45	&64.02$\pm$2.03		&	 --  \\\hline
&1	&2.86$\pm$0.89	&3.89$\pm$0.67	&2.49$\pm$1.01	&2.86$\pm$1.01	&\textbf{9.07}$\pm$0.83	&2.56$\pm$0.74	&\textbf{8.97}$\pm$0.97	&2.56$\pm$0.99	&7.21$\pm$1.35		\\
&5	&30.30$\pm$2.17	&21.89$\pm$2.04	&33.42$\pm$2.52	&33.59$\pm$2.63	&42.52$\pm$4.99	&35.18$\pm$1.32	&\textbf{43.32}$\pm$1.82	&35.18$\pm$2.52	&\textbf{44.65}$\pm$4.14		\\
&10	&48.97$\pm$2.96	&49.14$\pm$2.57	&45.65$\pm$2.85	&45.88$\pm$2.97	&63.16$\pm$5.32	&56.15$\pm$2.36	&\textbf{65.75}$\pm$2.76	&47.24$\pm$2.55	&\textbf{66.51}$\pm$3.10		\\
$6$&20	&58.57$\pm$2.84	&58.97$\pm$2.60	&59.73$\pm$2.96	&57.11$\pm$3.22	&70.73$\pm$5.39	&64.72$\pm$3.11	&\textbf{74.39}$\pm$2.79	&57.81$\pm$2.18	&\textbf{74.52}$\pm$3.10		\\
&50	&66.88$\pm$2.31	&67.84$\pm$2.48	&67.84$\pm$2.63	&64.05$\pm$2.95	&72.09$\pm$5.18	&65.32$\pm$2.86	&\textbf{74.75}$\pm$2.60	&65.12$\pm$2.92	&\textbf{75.12}$\pm$2.79		\\
&80	&{68.31}$\pm$2.33	&69.44$\pm$2.49	&\textbf{69.70}$\pm$2.35	&66.78$\pm$2.94	&-- 	&65.08$\pm$2.73	&\textbf{74.72}$\pm$2.56	&68.01$\pm$2.90		&	 --  \\\hline
&1	&2.68$\pm$1.19	&4.55$\pm$1.07	&2.12$\pm$1.07	&2.81$\pm$0.82	&\textbf{11.39}$\pm$1.85	&2.38$\pm$1.33	&\textbf{10.91}$\pm$1.37	&1.99$\pm$1.19	&8.57$\pm$1.68		\\
&5	&29.52$\pm$1.38	&22.51$\pm$1.29	&34.72$\pm$2.86	&32.03$\pm$2.46	&44.98$\pm$5.32	&35.58$\pm$2.04	&\textbf{45.89}$\pm$2.34	&35.02$\pm$1.82	&\textbf{45.93}$\pm$1.93		\\
&10	&51.21$\pm$2.11	&49.39$\pm$3.18	&46.19$\pm$2.43	&46.10$\pm$2.60	&65.67$\pm$5.82	&56.84$\pm$2.07	&\textbf{67.53}$\pm$2.34	&48.44$\pm$3.20	&\textbf{66.93}$\pm$2.01		\\
$7$&20	&59.57$\pm$2.64	&60.91$\pm$2.75	&61.69$\pm$2.57	&57.58$\pm$2.75	&73.16$\pm$4.28	&65.37$\pm$2.24	&\textbf{74.89}$\pm$1.97	&58.31$\pm$2.62	&\textbf{75.80}$\pm$2.24		\\
&50	&68.10$\pm$2.21	&69.35$\pm$1.89	&69.26$\pm$2.22	&65.54$\pm$2.79	&74.11$\pm$4.52	&65.37$\pm$2.02	&\textbf{75.24}$\pm$2.12	&66.06$\pm$3.10	&\textbf{76.88}$\pm$1.72		\\
&80	&69.70$\pm$2.84	&{70.39}$\pm$1.76	&\textbf{70.65}$\pm$1.97	&67.97$\pm$2.45	&-- 	&65.37$\pm$2.02	&\textbf{75.19}$\pm$2.18	&68.14$\pm$2.96		& --	  \\\hline
\end{tabular}
}\label{FERETtable1}
\end{table*}

\section{Experiments}
This section evaluates the proposed methods on both second-order and third-order tensor data in terms of recognition rate, the number of extracted features, captured variance, and convergence. In addition, we also study the effectiveness of the relaxed start strategy on other TVP-based PCA algorithms.

\textbf{Data:}\footnote{Both face and gait data are downloaded from: \url{http://www.dsp.utoronto.ca/~haiping/MSL.html}} For \emph{second-order} tensors, we use the same subset of the FERET database \cite{FERET00} as in \cite{lu2009uncorrelated}, with 721 face images from 70 subjects. Each face image is normalized to $80\times 60$ graylevel pixels. For \emph{third-order tensors}, we use a subset of the USF HumanID ``Gait Challenge'' database \cite{sarkar2005humanid}. We use the same gallery set (731 samples from 71 subjects) and probe A (727 samples from 71 subjects) as in \cite{lu2009uncorrelated}, and we also test probe B (423 samples from 41 subjects) and probe C (420 samples from 41 subjects). Each gait sample is a (binary) silhouette sequence of size of 32$\times$22$\times$10.

\textbf{Experiment setup:} In face recognition experiments, we randomly select $L = 1, 2, 3, 4, 5, 6, 7$ samples from each subject as the training data and use the rest for testing. We repeat such random splits (repetitions) ten times and report the mean correct recognition rates. In gait recognition experiments, we follow the standard setting and use the gallery set as the training data and probes A, B, and C as the test data (so there is no random splits/repetitions), and report the rank 1 and rank 5 recognition rates \cite{sarkar2005humanid}.

\textbf{Algorithms and their settings:} We first evaluate SO-MPCA and SO-MPCA-RS against five existing PCA-based methods: PCA \cite{pcabook}, CSA \cite{ConcurrentCVPR2005}, MPCA \cite{MPCATNN}, TROD \cite{TRODCVPR2001}, and UMPCA \cite{lu2009uncorrelated}.\footnote{For second-order tensors, CSA and MPCA are equivalent to GLRAM \cite{GPCAML05} and GPCA \cite{GPCAYe}, respectively.} CSA and MPCA produce tensorial features so they need to be vectorized. MPCA uses the full projection. For TROD and UMPCA, we use the uniform initialization \cite{lu2009uncorrelated}. For SO-MPCA and SO-MPCA-RS, we set the selected mode $\nu=1$ for the maximum number of features. For iterative algorithms, we set the number of iterations to 20. All features are sorted according to the scatters (captured variance) in descending order for classification. We use the \emph{Nearest Neighbor Classifier} with the Euclidean distance measure to classify the top $P$ features.  We test up to $P=80$ features in face recognition and up to $P=32$ features in gait recognition. The performance of FO-MPCA is much worse than SO-MPCA so it is not included in the comparisons (except variance study) to save space.

\begin{table*}[!ttt]\caption{Rank 1 and rank 5 gait recognition rates in percentage (mean $\pm$ std ) by the nearest neighbor classifier on the USF subset. The top two results are highlighted with bold fonts and `-' indicates that no enough features can be extracted.}
\centerline{
\scriptsize
\centering
\setlength{\tabcolsep}{7.5pt}
\renewcommand{\arraystretch}{1.55}
\begin{tabular}{|c|c|c|ccccc|cc| c c|}
\hline $Rank$ & $Probe$ & $  P  $   & PCA     &   CSA &   MPCA &  TROD &  UMPCA & SO-MPCA    & \textbf{SO-MPCA-RS }    & TROD-\textbf{RS }        & UMPCA-\textbf{RS} \\\hline%
& &5	&30.99	&22.54	&32.39	&28.17	&39.44	&30.99	&\textbf{40.85}	&18.31	&\textbf{39.44}	\\
&&10	&52.11	&43.66	&49.30	&42.25	&57.75	&49.30	&\textbf{59.15}	&33.80	&\textbf{63.38}	\\
&$A$&20	&\textbf{67.61}	&57.75	&60.56	&53.52	&--	&54.93	&\textbf{67.61}	&53.52	&--	\\
&&32	&\textbf{71.83}	&59.15	&61.97	&60.56	&--	&54.93	&\textbf{69.01}	&64.79	&	--  \\\cline{2-12}								
&&5	&26.83	&17.07	&24.39	&19.51	&26.83	&29.27	&\textbf{41.46}	&19.51	&\textbf{39.02}	\\
& &10	&48.78	&39.02	&46.34	&39.02	&46.34	&53.66	&\textbf{63.41}	&36.59	&\textbf{51.22}	\\
$1$&$B$&20	&\textbf{65.85}	&53.66	&58.54	&53.66	&--	&58.54	&\textbf{65.85}	&43.90	&--	\\
&&32	&\textbf{68.29}	&60.98	&58.54	&65.85	&--	&60.98	&\textbf{68.29}	&63.41	&	--  \\ \cline{2-12}																
&&5	&12.20	&9.76	&\textbf{14.63}	&4.88	&\textbf{24.39}	&12.20	&\textbf{14.63}	&7.32	&\textbf{14.63}	\\
&&10	&\textbf{29.27}	&14.63	&19.51	&14.63	&\textbf{29.27}	&\textbf{29.27}	&\textbf{29.27}	&19.51	&21.95	\\
&$C$&20	&\textbf{34.15}	&31.71	&29.27	&21.95	&--	&31.71	&\textbf{34.15}	&24.39	&--	\\
&&32	&\textbf{46.34}	&34.15	&29.27	&31.71	&--	&31.71	&\textbf{39.02}	&\textbf{39.02}	&--	  \\\hline
&	&5	&57.75	&56.34	&66.20	&54.93	&73.24	&67.61	&\textbf{84.51}	&42.25	&\textbf{73.24}	\\
&	&10	&80.28	&74.65	&77.46	&77.46	&83.10	&77.46	&\textbf{88.73}	&59.15	&\textbf{85.92}	\\
& $A$ &	20	&\textbf{87.32}	&80.28	&81.69	&76.06	&--	& 80.28	&\textbf{92.96}	& 77.46	&--	\\
&	&32	&\textbf{87.32}	&81.69	&83.10	&78.87	&--	&80.28	&\textbf{92.96}	&81.69	&--	\\\cline{2-12}								
&	&5	&48.78	&48.78	&53.66	&48.78	&{58.54}	&\textbf{63.41}	&\textbf{68.29}	&53.66	&{58.54}	\\
&	&10	&{73.17}	&70.73	&{70.73}	&60.98	&65.85	&70.73	&\textbf{75.61}	&\textbf{75.61}	&68.29	\\
$5$ &$B$	&20	&\textbf{78.05}	&75.61	&\textbf{78.05}	&75.61	&--	&70.73	&\textbf{80.49}	&\textbf{78.05}	&--	\\
&	&32	&{78.05}	&73.17	&{78.05}	&\textbf{80.49}	&--	&70.73	&\textbf{80.49}	&\textbf{80.49}	&--	\\			\cline{2-12}						
&	&5	&\textbf{51.22}	&34.15	&41.46	&36.59	&41.46	&43.90	&\textbf{56.10}	&29.27	&36.59	\\
&	&10	&53.66	&46.34	&43.90	&48.78	&43.90	&48.78	&\textbf{70.73}	&43.90	&\textbf{56.10}	\\
&$C$&20	&\textbf{65.85}	&56.10	&60.98	&48.78	&--	&48.78	&\textbf{78.05}	&46.34	&--	\\
&	&32	&\textbf{65.85}	&60.98	&58.54	&56.10	&--	&48.78	&\textbf{78.05}	&56.10	&	--  \\\hline
\end{tabular}
}\label{FERETtable2}
\end{table*}

\textbf{Face recognition results:} Table \ref{FERETtable1} shows the face recognition results for $P= 1, 5, 10, 20, 50, 80$ and $L=1, 2, 3, 4, 5, 6, 7$, including both the mean and the standard deviation (std) over ten repetitions. We highlight the top two results in each row in bold fonts for easy comparison. Only SO-MPCA-RS consistently achieves the top 2 results in all cases. Compared with existing methods (PCA, CSA, MPCA, TROD and UMPCA), SO-MPCA-RS outperforms the best performing existing algorithm (UMPCA) by 3.79$\%$ on average.
 	
Furthermore, for larger $L = 5, 6, 7$, SO-MPCA-RS outperforms the other five methods at least by 2.26$\%$ on average. For smaller $L = 1, 2, 3$, SO-MPCA-RS achieves a greater improvement of at least 5.28$\%$ over existing methods, indicating that SO-MPCA-RS is more superior in dealing with the small sample size (overfitting) problem.

\textbf{Gait recognition results:} Similarly, the gait recognition results are reported in Table \ref{FERETtable2} with the top two results  highlighted. Again, only SO-MPCA-RS consistently achieves the top 2 results in all cases. In rank 1 rate, the best performing existing algorithm is PCA, which outperforms SO-MPCA-RS by 3.73$\%$ on average. While in rank 5 rate, SO-MPCA-RS outperforms the best performing existing algorithm (still PCA) by 6.76$\%$ on average.

\textbf{Number of features:} In the tables, we use `-' to indicate that there are not enough features. For PCA, there are at most 69 features for face data when $L=1$ since there are only 70 samples for training. UMPCA can only extract 60 or 10 features for face and gait data, respectively. In contrast, SO-MPCA and SO-MPCA-RS (with $\nu=1$) can learn 80 features for face data and 32 features for gait data. 

\textbf{Feature variance:} We illustrate the variance captured by PCA, UMPCA, FO-MPCA, SO-MPCA, and SO-MPCA-RS in Fig. \ref{Facetest1} for face data with $L=1$ (not all methods are shown for clarity). Figure \ref{VarianceST} shows the sorted variance. It is clear that semi-orthogonality captures more variance than full-orthogonality, as we discussed in Sec. 3.1. Moreover, both  SO-MPCA and SO-MPCA-RS can capture more variance than UMPCA, but less than PCA (and also CSA, MPCA, and TROD, which are not shown). Though capturing less variance, SO-MPCA-RS achieves better overall classification performance than other PCA-based methods, with results consistently in the top two in all experiments.
\begin{figure} [!ttt]
\centering\subfigure[Sorted variance ]{%
{\includegraphics[width=5.0cm]{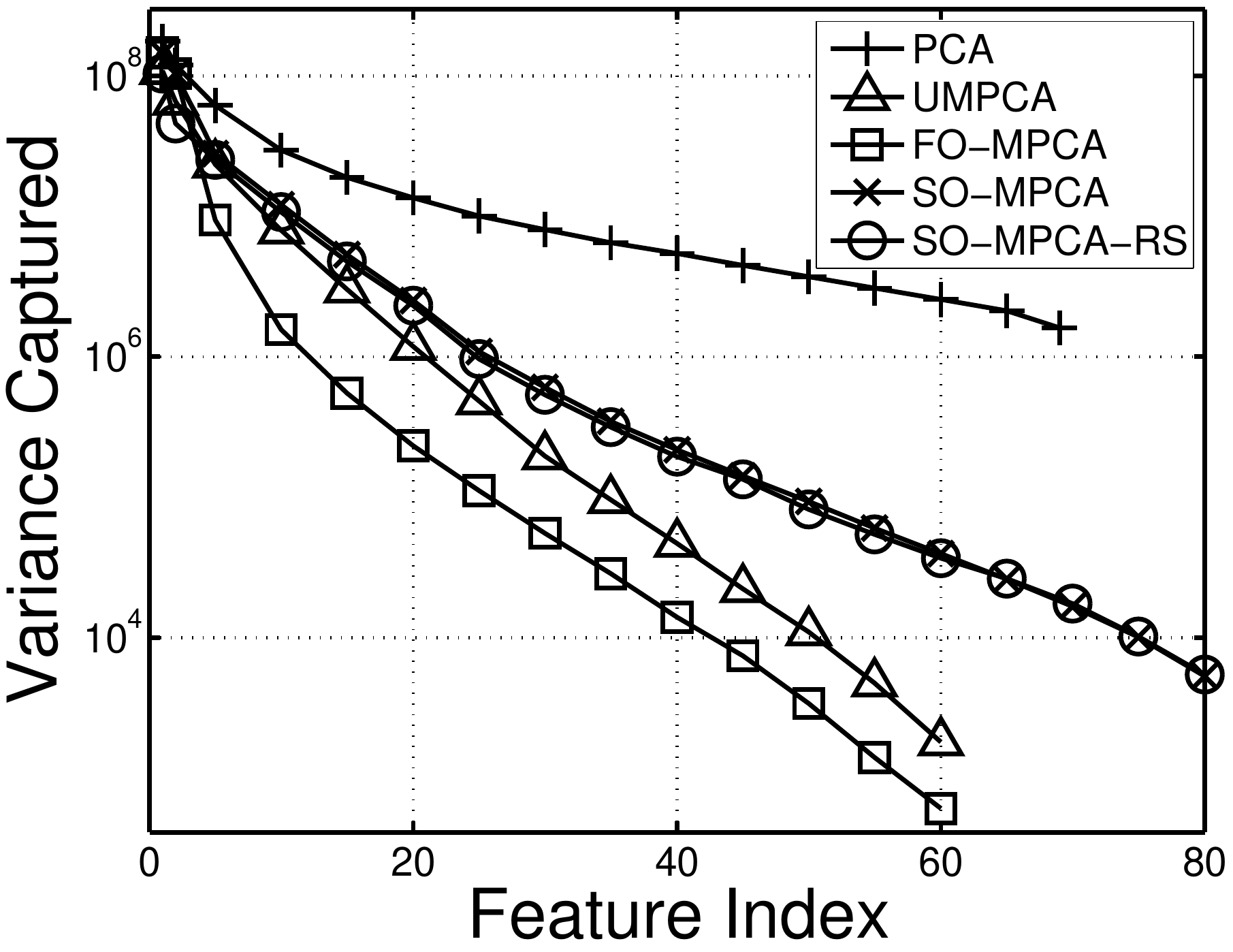} \label{VarianceST}}}
\centering\subfigure[Unsorted variance ]{%
{\includegraphics[width=5.0cm]{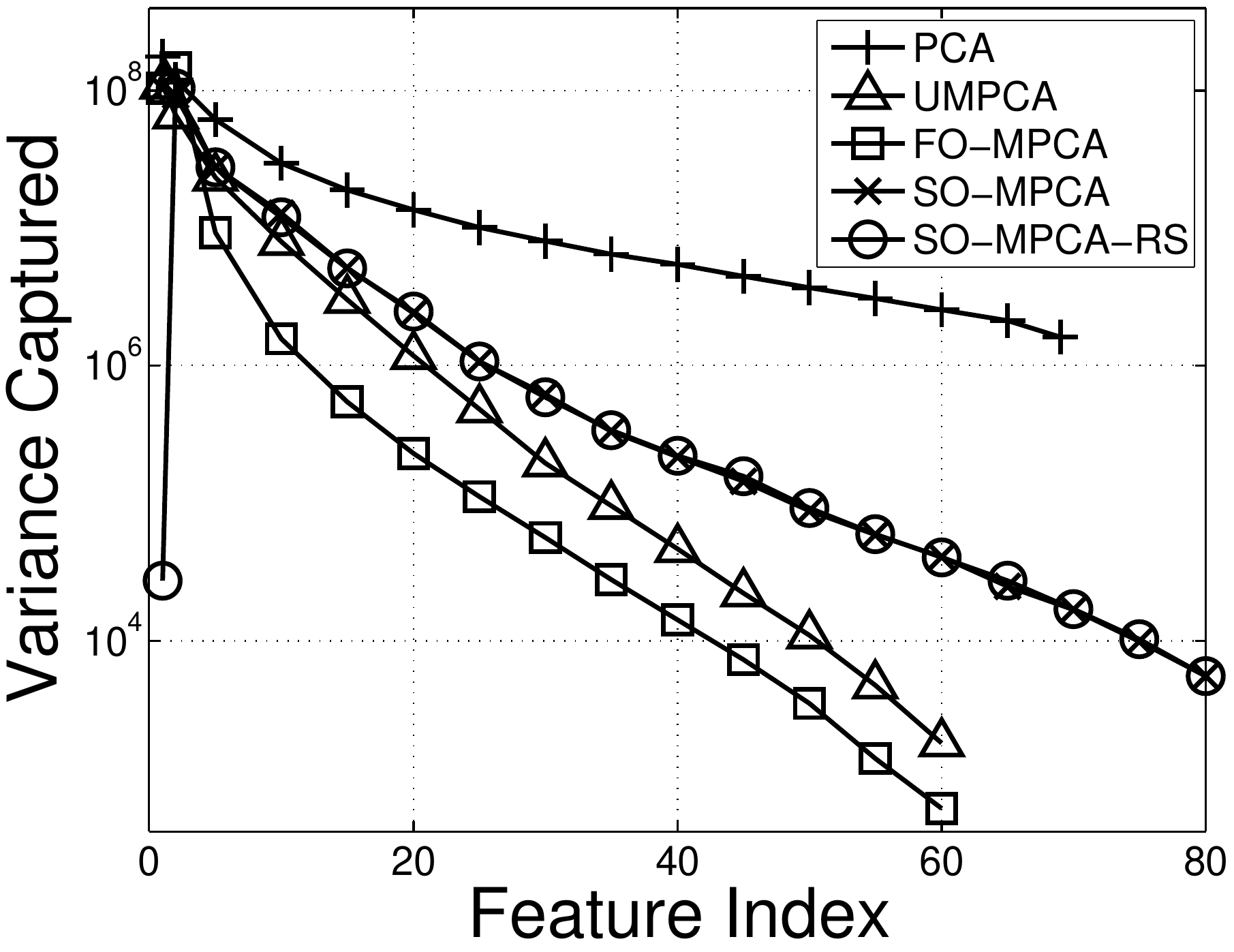} \label{VarianceOUTST}}}
\caption{The captured variance on face data with $L=1$.}
\label{Facetest1}%
\end{figure}

\begin{figure} [!ttt]
\centering\subfigure[$p=2$]{
{\includegraphics[width=4.1cm]{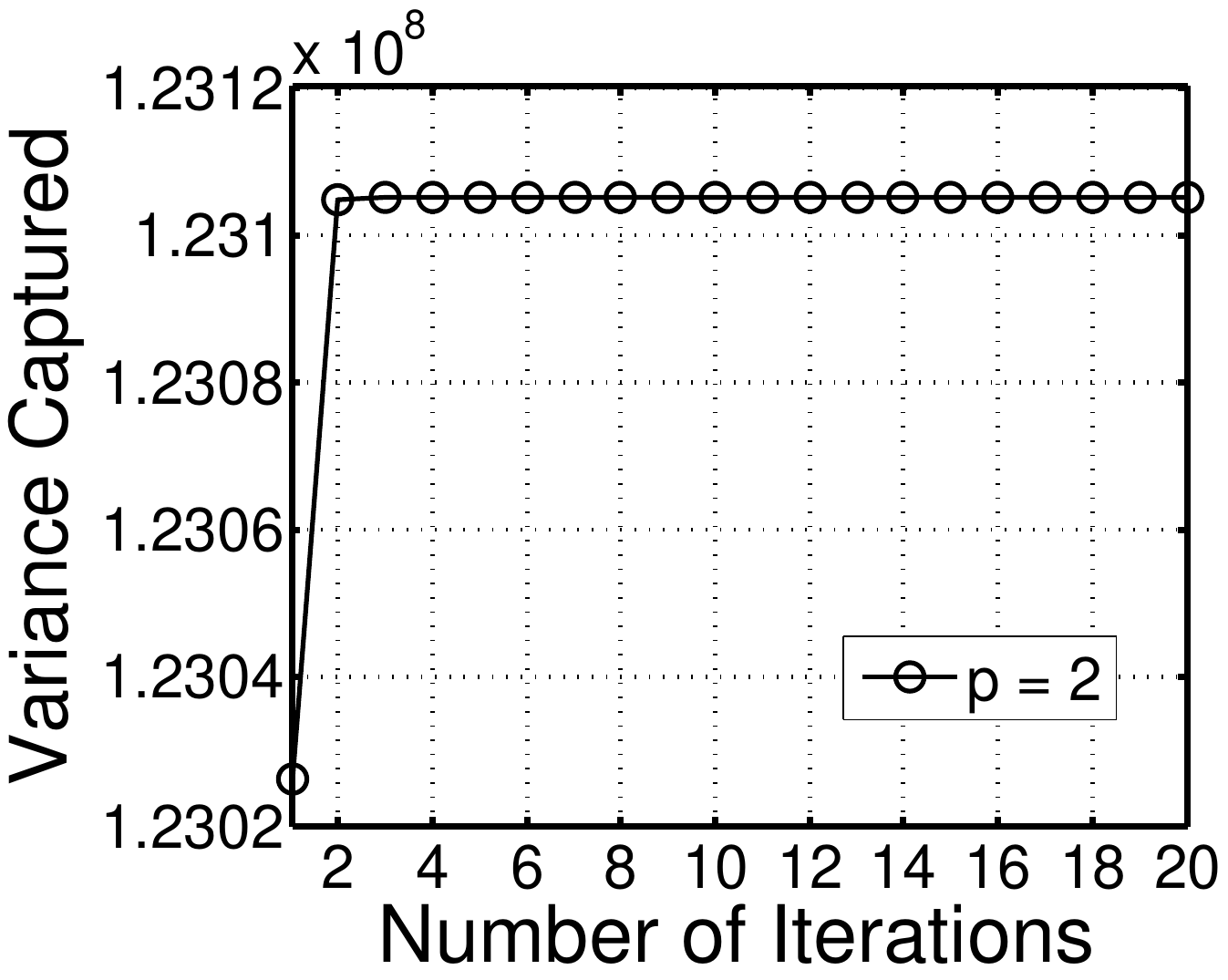}\label{convergence5}}}
\centering\subfigure[$p=5$]{
{\includegraphics[width=4.1cm]{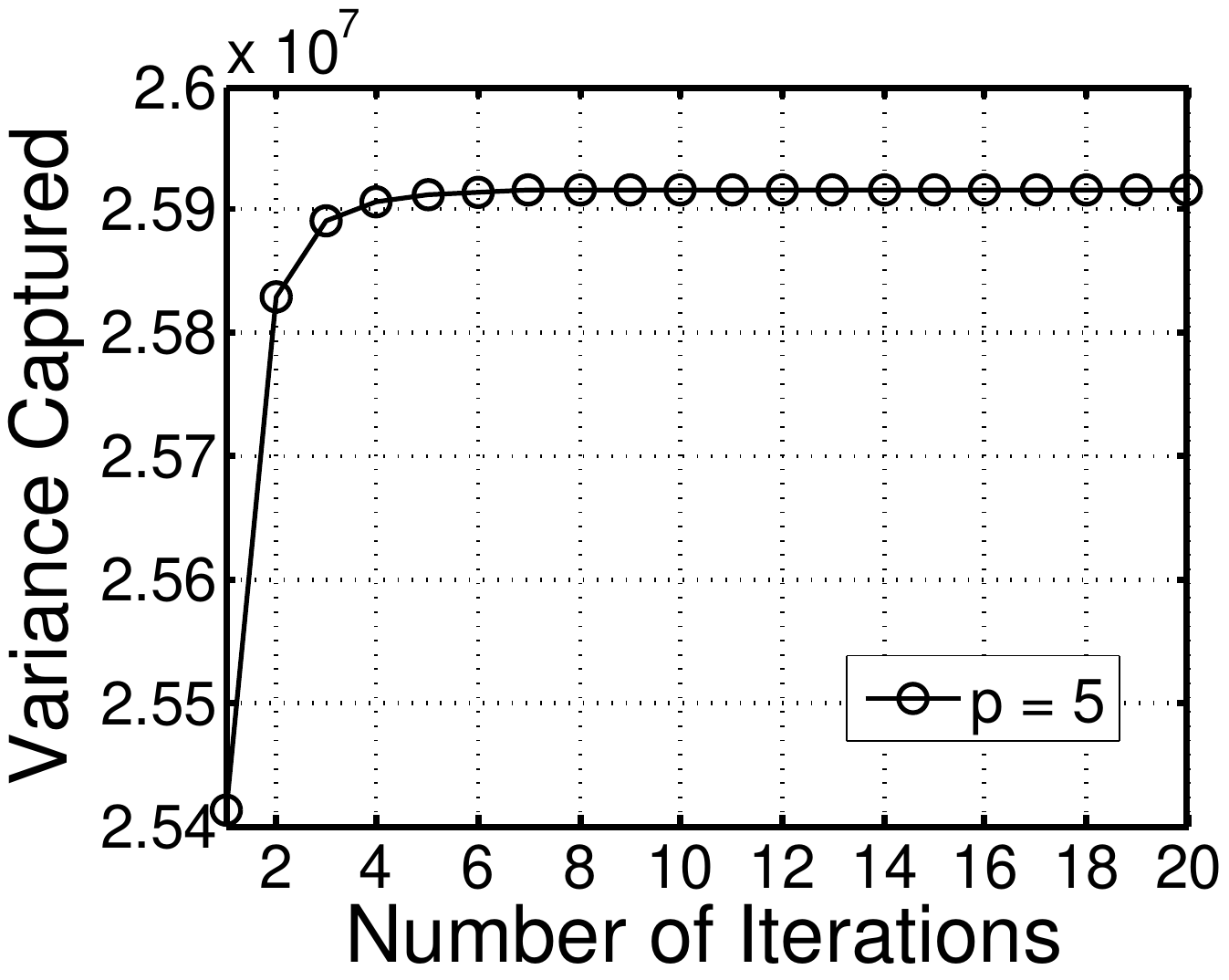}\label{convergence15}}}
\caption{Illustration of the SO-MPCA-RS algorithm's convergence performance on the face data with $L=1$.}
\label{convergface}
\end{figure}

We also show the unsorted captured variance in Fig. \ref{VarianceOUTST}. The variance captured by the first (fixed) EMP of SO-MPCA-RS is much less than other EMPs, which is not surprising since the variance is not maximized.

\textbf{Convergence:} We demonstrate the convergence of SO-MPCA-RS in Fig. \ref{convergface} for face data with $L=1$. We can see that SO-MPCA-RS converges in just a few iterations. SO-MPCA has a similar convergence rate.

\textbf{Effectiveness of relaxed start:} To evaluate the proposed relaxed start strategy, we apply it to two other  TVP-based methods, TROD and UMPCA, getting TROD-RS and UMPCA-RS, respectively. We summarize their performance in the last two columns of Tables \ref{FERETtable1} and \ref{FERETtable2} for face and gait recognition experiments, respectively.

Both tables show that relaxed start can help both TROD and UMPCA to achieve better recognition rates. From Table \ref{FERETtable1}, TROD-RS improves over TROD by 0.32$\%$ and UMPCA-RS improves over UMPCA by 2.15$\%$ on average. From Table \ref{FERETtable2}, TROD-RS achieves 3.03 $\%$ improvement over TROD and UMPCA-RS achieves 1.07$\%$ improvement over UMPCA for rank 1 rate. For rank 5 rate, TROD-RS improves 0.94$\%$ over TROD, and  UMPCA-RS improves 5.28$\%$ over UMPCA. The relaxed start is most effective for our SO-MPCA. SO-MPCA-RS has an improvement of 9.97$\%$ on face data over SO-MPCA. On gait data, SO-MPCA-RS outperforms SO-MPCA by 9.56$\%$ in rank 1 rate and 17.26$\%$ in rank 5 rate on average.

In addition, SO-MPCA-RS has better face recognition performance than TROD-RS and UMPCA-RS with an improvement of 8.08$\%$ and 1.64$\%$, respectively. On gait data,  SO-MPCA-RS improves the rank 1 recognition rate by 3.03$\%$ over  TROD-RS and 13.25$\%$ over UMPCA-RS  on average, and SO-MPCA-RS improves the rank 5 recognition rate by 11.07$\%$ and 13.73$\%$ over  TROD-RS  and  UMPCA-RS.

These controlled experiments  show the effectiveness of relaxed start on SO-MPCA and other TVP-based multilinear PCA methods (UMPCA and TROD). One possible explanation is that RS increases the bias and reduces the variance of the learning model, while further investigation is needed.

\section{Conclusion}
This paper proposes a novel multilinear PCA algorithm under the TVP setting, named as semi-orthogonal multilinear PCA with relaxed start (SO-MPCA-RS). The proposed SO-MPCA approach learns features directly from tensors via TVP to maximize the captured variance with the orthogonality constraint imposed in only one mode. This semi-orthogonality can capture more variance and learn more features than full-orthogonality. Furthermore, the introduced relaxed start strategy can achieve better generalization by fixing the starting projection vectors to uniform vectors to increase the bias and reduce the variance of the learning model. Experiments on face (2D data) and gait (3D data) recognition show that SO-MPCA-RS achieves the best overall performance compared with competing algorithms. In addition, relaxed start is also effective for other TVP-based PCA methods.

In this paper, we studied semi-orthogonality in only one mode. A possible future work is to learn SO-MPCA-RS features from each mode separately and then do a feature/score-level fusion.

\section*{Acknowledgments}
This research was supported by Research Grants Council of the Hong Kong Special Administrative Region (Grant 22200014 and the Hong Kong PhD Fellowship Scheme).

\bibliographystyle{named}
\bibliography{ijcai15SOMPCARS}

\begin{thebibliography}{}

\bibitem[\protect\citeauthoryear{Abu-Mostafa \bgroup \em et al.\egroup
  }{2012}]{abu2012learning}
Y.~S. Abu-Mostafa, M.~Magdon-Ismail, and H.-T. Lin.
\newblock {\em Learning from Data}, chapter~4, pages 119--166.
\newblock AMLBook, 2012.

\bibitem[\protect\citeauthoryear{Anaraki and Hughes}{2014}]{anaraki2014memory}
F.P. Anaraki and S.~Hughes.
\newblock Memory and computation efficient {PCA} via very sparse random
  projections.
\newblock In {\em Proc. 31st Int. Conf. on Machine Learning}, pages 1341--1349,
  2014.

\bibitem[\protect\citeauthoryear{Deng \bgroup \em et al.\egroup
  }{2014}]{deng2014transform}
W.~Deng, J.~Hu, J.~Lu, and J.~Guo.
\newblock Transform-invariant {PCA}: A unified approach to fully automatic face
  alignment, representation, and recognition.
\newblock {\em IEEE Trans. Pattern. Anal. Mach. Intell.}, 36(6):1275--1284,
  2014.

\bibitem[\protect\citeauthoryear{Edelman \bgroup \em et al.\egroup
  }{1998}]{edelman1998geometry}
A.~Edelman, T.A. Arias, and S.T. Smith.
\newblock The geometry of algorithms with orthogonality constraints.
\newblock {\em SIAM J. Matrix Anal. Appl.}, 20(2):303--353, 1998.

\bibitem[\protect\citeauthoryear{Faloutsos \bgroup \em et al.\egroup
  }{2007}]{KoldaICML07TensorTut}
C.~Faloutsos, T.~G. Kolda, and J.~Sun.
\newblock Mining large time-evolving data using matrix and tensor tools.
\newblock In {\em Int. Conf. on Data Mining Tutorial}, 2007.

\bibitem[\protect\citeauthoryear{Gao \bgroup \em et al.\egroup
  }{2013}]{gao2013stable}
Q.~Gao, J.~Ma, H.~Zhang, X.~Gao, and Y.~Liu.
\newblock Stable orthogonal local discriminant embedding for linear
  dimensionality reduction.
\newblock {\em IEEE Trans. Image Processing}, 22(7):2521--2531, 2013.

\bibitem[\protect\citeauthoryear{Harshman}{1970}]{PARAFAC70}
R.~A. Harshman.
\newblock Foundations of the parafac procedure: Models and conditions for an
  ``explanatory'' multi-modal factor analysis.
\newblock {\em UCLA Working Papers in Phonetics}, 16:1--84, 1970.

\bibitem[\protect\citeauthoryear{Hua \bgroup \em et al.\egroup
  }{2007}]{hua2007face}
G.~Hua, P.A. Viola, and S.M. Drucker.
\newblock Face recognition using discriminatively trained orthogonal rank one
  tensor projections.
\newblock In {\em Proc. {IEEE} Int. Conf. on Computer Vision and Pattern
  Recognition}, pages 1--8, 2007.

\bibitem[\protect\citeauthoryear{Jolliffe}{2002}]{pcabook}
I.~T. Jolliffe.
\newblock {\em Principal Component Analysis}.
\newblock Springer Series in Statistics, second edition, 2002.

\bibitem[\protect\citeauthoryear{Ke and Sukthankar}{2004}]{ke2004pca}
Y.~Ke and R.~Sukthankar.
\newblock {PCA-SIFT}: A more distinctive representation for local image
  descriptors.
\newblock In {\em Proc. {IEEE} Int. Conf. on Computer Vision and Pattern
  Recognition}, volume~II, pages 506--513, 2004.

\bibitem[\protect\citeauthoryear{Kokiopoulou and
  Saad}{2007}]{kokiopoulou2007orthogonal}
E.~Kokiopoulou and Y.~Saad.
\newblock Orthogonal neighborhood preserving projections: A projection-based
  dimensionality reduction technique.
\newblock {\em IEEE Trans. Pattern. Anal. Mach. Intell.}, 29(12):2143--2156,
  2007.

\bibitem[\protect\citeauthoryear{Kolda and Bader}{2009}]{KoldaReview09}
T.~G. Kolda and B.~W. Bader.
\newblock Tensor decompositions and applications.
\newblock {\em SIAM Rev.}, 51(3):455--500, 2009.

\bibitem[\protect\citeauthoryear{Kolda}{2001}]{kolda2001orthogonal}
T.G. Kolda.
\newblock Orthogonal tensor decompositions.
\newblock {\em SIAM J. Matrix Anal. Appl.}, 23(1):243--255, 2001.

\bibitem[\protect\citeauthoryear{Kusner \bgroup \em et al.\egroup
  }{2014}]{kusner2014stochastic}
M.~Kusner, S.~Tyree, K.Q. Weinberger, and K.~Agrawal.
\newblock Stochastic neighbor compression.
\newblock In {\em Proc. 31st Int. Conf. Machine Learning}, pages 622--630,
  2014.

\bibitem[\protect\citeauthoryear{Lathauwer \bgroup \em et al.\egroup
  }{2000}]{MultilinearBestR00}
L.~De Lathauwer, B.~De Moor, and J.~Vandewalle.
\newblock On the best rank-1 and rank-(${R}_1,{R}_2,...,{R}_{N}$) approximation
  of higher-order tensors.
\newblock {\em SIAM J. Matrix Anal. Appl.}, 21(4):1324--1342, 2000.

\bibitem[\protect\citeauthoryear{Lu \bgroup \em et al.\egroup }{2008}]{MPCATNN}
H.~Lu, K.~N. Plataniotis, and A.~N. Venetsanopoulos.
\newblock {MPCA}: Multilinear principal component analysis of tensor objects.
\newblock {\em IEEE Trans. Neural Networks}, 19(1):18--39, 2008.

\bibitem[\protect\citeauthoryear{Lu \bgroup \em et al.\egroup
  }{2009}]{lu2009uncorrelated}
H.~Lu, K.~N. Plataniotis, and A.~N. Venetsanopoulos.
\newblock Uncorrelated multilinear principal component analysis for
  unsupervised multilinear subspace learning.
\newblock {\em IEEE Trans. Neural Networks}, 20(11):1820--1836, 2009.

\bibitem[\protect\citeauthoryear{Lu \bgroup \em et al.\egroup
  }{2013}]{lu2013multilinear}
H.~Lu, K.~N. Plataniotis, and A.~N. Venetsanopoulos.
\newblock {\em Multilinear Subspace Learning: Dimensionality Reduction of
  Multidimensional Data}.
\newblock CRC Press, 2013.

\bibitem[\protect\citeauthoryear{Phillips \bgroup \em et al.\egroup
  }{2000}]{FERET00}
P.~J. Phillips, H.~Moon, S.~A. Rizvi, and P.~Rauss.
\newblock The {FERET} evaluation method for face recognition algorithms.
\newblock {\em IEEE Trans. Pattern. Anal. Mach. Intell.}, 22(10):1090--1104,
  2000.

\bibitem[\protect\citeauthoryear{Sarkar \bgroup \em et al.\egroup
  }{2005}]{sarkar2005humanid}
S.~Sarkar, P.J. Phillips, Z.~Liu, I.~R. Vega, P.~Grother, and K.~W.Bowyer.
\newblock The human {ID} gait challenge problem: Data sets, performance, and
  analysis.
\newblock {\em IEEE Trans. Pattern. Anal. Mach. Intell.}, 27(2):162--177, 2005.

\bibitem[\protect\citeauthoryear{Shashua and Levin}{2001}]{TRODCVPR2001}
A.~Shashua and A.~Levin.
\newblock Linear image coding for regression and classification using the
  tensor-rank principle.
\newblock In {\em Proc. {IEEE} Int. Conf. on Computer Vision and Pattern
  Recognition}, volume~I, pages 42--49, 2001.

\bibitem[\protect\citeauthoryear{Wang \bgroup \em et al.\egroup
  }{2015}]{wang2015orthogonal}
L.~Wang, M.T. Chu, and B.~Yu Wang.
\newblock {\em SIAM J. Matrix Anal. Appl.}, 36(1):1--19, 2015.

\bibitem[\protect\citeauthoryear{Xu \bgroup \em et al.\egroup
  }{2005}]{ConcurrentCVPR2005}
D.~Xu, S.~Yan, L.~Zhang, H.-J. Zhang, Z.~Liu, and H.-Y. Shum;.
\newblock Concurrent subspaces analysis.
\newblock In {\em Proc. {IEEE} Int. Conf. on Computer Vision and Pattern
  Recognition}, volume~II, pages 203--208, 2005.

\bibitem[\protect\citeauthoryear{Yang \bgroup \em et al.\egroup
  }{2004}]{2DPCApami04}
J.~Yang, D.~Zhang, A.~F Frangi, and J.~Yang.
\newblock Two-dimensional {PCA}: a new approach to appearance-based face
  representation and recognition.
\newblock {\em IEEE Trans. Pattern. Anal. Mach. Intell.}, 26(1):131--137, 2004.

\bibitem[\protect\citeauthoryear{Ye \bgroup \em et al.\egroup }{2004}]{GPCAYe}
J.~Ye, R.~Janardan, and Q.~Li.
\newblock {GPCA}: An efficient dimension reduction scheme for image compression
  and retrieval.
\newblock In {\em Proc. ACM SIGKDD Int. Conf. on Knowledge Discovery and Data
  Mining}, pages 354--363, 2004.

\bibitem[\protect\citeauthoryear{Ye}{2005}]{GPCAML05}
J.~Ye.
\newblock Generalized low rank approximations of matrices.
\newblock {\em Machine Learning}, 61(1-3):167--191, 2005.

\end{thebibliography}
\end{document}